\documentclass[review]{siamart0216}

\usepackage{lipsum}
\usepackage{amsfonts}
\usepackage{graphicx}

\usepackage{epstopdf}
\usepackage{algorithmic}

\usepackage{tipa}

\usepackage{amssymb}
\usepackage{enumitem}

\ifpdf
  \DeclareGraphicsExtensions{.eps,.pdf,.png,.jpg}
\else
  \DeclareGraphicsExtensions{.eps}
\fi

\newcommand{\TheTitle}{Stability and Fluctuations in a Simple Model of  Phonetic Category Change} 
\newcommand{\TheAuthors}{B. Goodman, and P. F. Tupper}

\headers{A Simple Model of  Phonetic Category Change}{\TheAuthors}

\title{{\TheTitle}\thanks{This work was funded by the National Science and Engineering Research Council, Canada.}}

\author{
  Benjamin Goodman\thanks{Department of Mathematics, Simon Fraser University, Burnaby, BC, Canada \newline
    (\email{bgoodman@sfu.ca, pft3@sfu.ca}).}
  \and
  Paul F. Tupper\footnotemark[2]}

\usepackage{amsopn}

\newcommand{\R}{\mathbb{R}}

\newcommand{\beq}{\begin{eqnarray*}}
\newcommand{\eeq}{\end{eqnarray*}}
\newcommand{\beqa}{\begin{eqnarray}}
\newcommand{\eeqa}{\end{eqnarray}}
\newcommand{\eps}{\varepsilon}

\newcommand{\expect}{\mathbb{E}}
\newcommand{\Var}{\mathrm{Var}}

\def\brac#1{\left( #1 \right)}
\def\sqbrac#1{\left[ #1 \right]}
\def\squigbr#1{\left\{ #1 \right\}}

\DeclareMathOperator*{\argmax}{arg\,max}

\begin{document}


\maketitle

\begin{abstract}
In spoken languages, speakers divide up the space of phonetic possibilities into different regions, corresponding to different phonemes. We consider a simple exemplar model of how this division of phonetic space varies over  time among a population of language users. In the particular model we consider,
we show that, once the system is initialized with a given set of phonemes, that phonemes do not become extinct: all phonemes will be maintained in the system for all time.
 This is  in contrast to what is observed in more complex models. Furthermore, we show that the boundaries between phonemes fluctuate and we quantitatively study the fluctuations in a simple instance of our model.  These results prepare the ground for more sophisticated models in which some phonemes go extinct or new phonemes emerge through other processes.
\end{abstract}

\newtheorem{thm}{Theorem}
\newtheorem{cor}{Corollary}

\begin{keywords}
  exemplar models, linguistics, categorization, $k$-means clustering, random dynamical systems
\end{keywords}

\begin{AMS}
  91F20, 94A99, 60GNN
\end{AMS}

\section{Introduction}

Exemplar models are  used in linguistics to describe how language users store linguistic categories \cite{wedel2006}.  Examples of the type of linguistic categories we have in mind are vowel sounds like \textipa{A}, \textipa{3}, \textipa{I} (corresponding to the vowel sounds in `bat', `bet', `bit', respectively). When a person hears a vowel sound within a word, they have to classify it as belonging to one of the categories of vowels based on the sound's acoustic properties. This classification will determine what word the person understands is being uttered, e.g. `tack', `tech', or `tick'.  An important issue in linguistics is how  language users perform this classification.

Exemplar models provide one answer to this question \cite{pierrehumbert}. According to exemplar models, every member of a linguistic community stores a multitude of  detailed memories of every sound that they hear. These memories are called exemplars.  Exemplars consist of detailed acoustic information about the sound, as well as a category label: information about what the sound is classified as.  For example, with the case of vowels, according to exemplar models, every person holds a detailed memory of every vowel they have ever heard, labeled with the corresponding vowel category, \textipa{A}, \textipa{3}, \textipa{I}, etc \cite{jaeger1}. 

Exemplar models provide a theory of both perception and production. When the language user hears a new vowel sound, the sound is compared to other exemplars already in memory and is classified according to the labels of exemplars that it is close to. When the language user needs to produce a new instance of a vowel, they select an exemplar from the set of all exemplars with the appropriate label and utter a copy of it, usually with either noise or bias added.

An important feature of many exemplar models is that exemplars do not remain in memory unchanged forever. 
A popular choice is for each exemplar to have a weight that decays with time \cite{pierrehumbert,wedel2012}. These exemplar weights enter into both perception and production, indicating that certain exemplars are more important for the relevant process than others.  New exemplars are created (every time a new instance is perceived) with some large weight which then decays exponentially with time.  This allows old exemplars to be forgotten and the general population of exemplars in a language user's mind to change with time.


There has been relatively little mathematical analysis of exemplar models. In \cite{tupper2015} the author studies a fairly elaborate exemplar model that is able to account for the phenomenon of sound merger.
The author is able to obtain analytical results by looking at a certain limiting case of the model, a limit in which there are no stochastic fluctuations. In \cite{tupper2015} perceptual boundaries, where language users switch from classifying a stimulus as one sound versus another, approach a stable configuration in perceptual space.
Our intent here is to study the fluctuations of perceptual boundaries within an exemplar model.

Our starting point is an exemplar model that was studied by MacQueen in 1967, originally as an algorithm for $k$-means clustering \cite{macqueen}. We can put MacQueen's work in the context of exemplar models as follows. Suppose an individual has a phonetic space (that is, a space of possible sounds) with $k$ labeled exemplars, one for each of  $k$ categories.   Suppose the individual receives an independent identically distributed (i.i.d.)\ sequence of acoustic inputs that they have to classify into these $k$ categories.  Rather than the criteria for categories being pre-given, the classification is performed on the fly using the exemplars that are already stored. Thus as more exemplars are stored the criteria for classification changes. If we assume that (i) the weights of the exemplars are all equal and do not change with time, (ii) for each category the mean acoustic value of all exemplars in that category is stored, (iii) new exemplars are classified according to which category mean they are closest to, we obtain the MacQueen model. 

To explain MacQueen's \cite{macqueen} results, we recall the definition of a centroidal Voronoi tessellation \cite{voronoi}. Given a set of \emph{generators}, which are just a finite set of points in the space, the Voronoi tessellation is a partitioning of the space where each point is assigned to a cell based on which generator it is closest to. A centroidal Voronoi tessellation of a region is a Voronoi tessellation in which each generator is the centroid (i.e.\ the center of mass) of its cell. Centroidal voronoi tessellations have already been established as being fundamental in some game theoretic models of language \cite{jaeger2}.
MacQueen's result is that in his model the distance between the category means and generators of  centroidal Voronoi tessellations converge to 0.
This implies in turn that the perceptual boundaries of the language-users align with the boundaries of  centroidal Voronoi tessellations. 

Taken as an exemplar model, the MacQueen model deviates from more realistic models of language use and development in several ways. In order to move in the direction of analyzing more realistic exemplar models, our contribution in this paper is to introduce weight decay into MacQueen's model. In our model every exemplar starts with weight 1 and the weight then exponentially decays with time. As we will show, this causes the Voronoi regions to no longer settle down into a stable configuration, but instead randomly fluctuate for all time; this is the main result of Section~\ref{sec:thm1}. We then consider a simple special case of our model for which we perform a quantitative analysis of the motion of the perceptual boundary between two categories.

In Section~\ref{secdef} we formally specify the exemplar model  model we study here.  In Section~\ref{sec:thm1} we go on to investigate a number of properties of the model, and most significantly, prove that when we have decay of exemplar weights then the exemplar means do not converge.  In  Section~\ref{sec:twocbeh} we go on to study our model in a special one-dimensional case with two categories. We provide a probabilistic model for the motion of the perceptual boundary.

\section{The $k$-Means Exemplar Model}
\label{secdef}

We imagine a language user who hears a sequence of sounds and classifies each of the sounds into one of $k$ categories, $k \geq 2$. The acoustic properties of the sound heard at the $n$th time step are denoted by $z_n \in \mathbb{R}^N$. We assume that all $z_n$ lie in a set $E \subseteq \mathbb{R}^N$ that is bounded, convex, closed, and has a non-empty interior.
$E$ corresponds to the space of all phonetically possible sounds.  
The sounds $z_1, z_2, \ldots$ are generated in $E$ independently according to a fixed probability measure $P$. 
  We assume the probability measure $P$ can be written as
\begin{equation} \label{eqn:Pdef}
P(A)=\int_{A} f(x) d x 
\end{equation}
for all measurable $A \subseteq E$, where $f(x)>0$, for all $x\in E$

\begin{figure}[htbp]
  \centering
  \includegraphics[width=13cm]{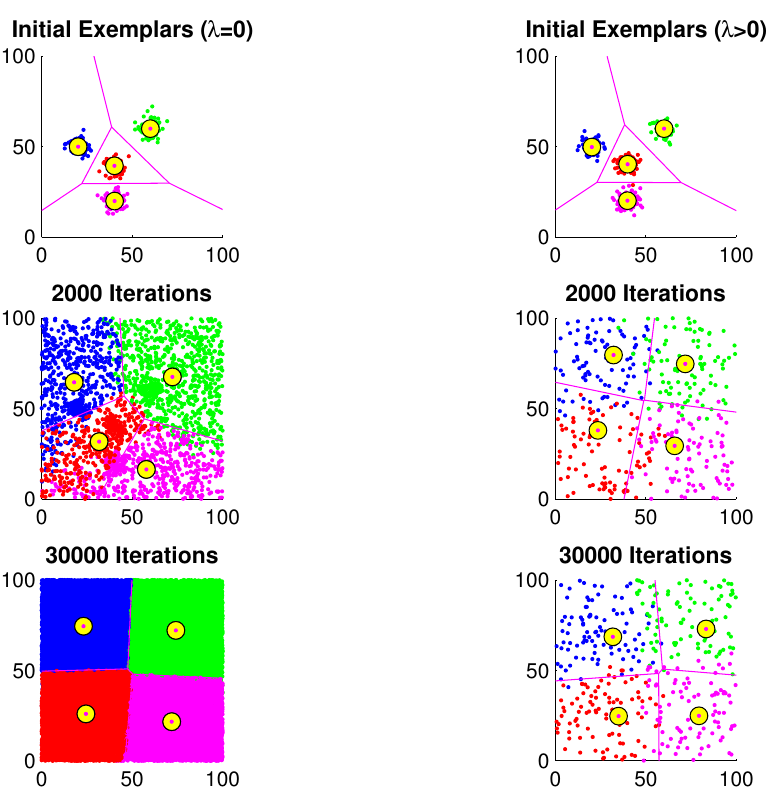}
  \caption{The exemplar dynamics model for $\lambda=0$ (left) and $\lambda=0.05$ (right). Small dots indicate individual exemplars, color-coded to indicate which category they were classified into.  All dots with weight larger than $0.01$ are plotted. The larger yellow dots indicate category means. Magenta lines show the boundaries between the Voronoi regions defined by the  category means. }
  \label{fig:2dmodel_demonstration}
\end{figure}

At the start of the model, we imagine that our language user already has a number of exemplars (each with a corresponding value in $E$) in each of the $k$ categories.
At time $n$, in category $j=1,\ldots,k$, the language user has exemplars with phonetic values  $a_j^{i} \in E$ and weights $v_j^i$ for $i = 1, \ldots, n_j$, where $n_j$ is the number of exemplars in category $j$.  At time step $n$,  a new sound with phonetic parameters $z_n \in E$ is heard.
This sound is stored as an exemplar in one of the language user's categories. Which category the new sound is stored in depends on the category means of all the categories. We define the category means at time step $n$, $x_j^n$ for $j=1,\ldots,k$ to be
\[
x_j^n = \frac{ \sum_{i=1}^{n_j} v_j^i a_i^i }{\sum_{i=1}^{n_j} v_j^i}
\]
and the total category weights by
\[
w_j^n = \sum_{i=1}^{n_j} v_j^i.
\]
We assign the new phonetic value $z_n$ to be an exemplar of category $j$ if for all $\ell$ 
\[
| z_n - x_j^n | \leq | z_n - x_\ell^{n} |.
\]
(In the case of a tie, the exemplar is assigned to the category with the lower index.) The new exemplar is always given an initial weight of 1.

We can specify the update procedure in another way as follows:
New exemplars entering the system are classified according to which Voronoi cell they are in according to the Voronoi tesselation generated by the category means.
In other words, exemplars are assigned to the category mean they are closest to.
 More formally,  if the category means are  $x^n_1, \ldots, x^n_k$
we define the Voronoi cell $S_{i}(x^{n})$ as the set
\begin{equation} \label{eqn:voronoidef}
\begin{split}
S_{i}(x^{n})=\{ \xi :  & \xi \in E, |\xi - x_{i}^{n}| \leq |\xi - x_{j}^{n}| \mbox{ for }j=1,2,\dots,k,  \\
                                 & \mbox{ and } \xi \notin S_{m}(x^{n}) \mbox{, for all } m<i. \}
\end{split}
\end{equation}
$S_{i}(x^{n})$ contains points in $E$ closest to $x_{i}^{n}$ with tied points being assigned to the lower index.  We will let $S_{i}^{n}=S_{i}(x^{n})$ for convenience of notation.

Our system will evolve as an iterative process in the following way.  At each step all weights in the system decay, which we model by multiplying weights by $e^{-\lambda}$, where $\lambda$ is a positive parameter.  In addition, if $z_{n}\in S_{i}^{n}$ then we update the value of the average $x^n_i$ by including $z_n$ in it, and also by increasing the total weight of the category by $1$. So if $z_n \in S_i^n$ we set
\begin{equation*} 
x_{i}^{n+1}=\dfrac{x_{i}^{n} w_{i}^{n} e^{-\lambda}+z_{n}}{w_{i}^{n}e^{-\lambda }+1},   \ \ \ \    w_{i}^{n+1}=w_{i}^{n}e^{-\lambda }+1,
\end{equation*}
\noindent and for all $j\neq i$ we set 
\begin{equation*} 
x_{j}^{n+1}=x_{j}^{n},   \ \ \ \ \  w_{j}^{n+1}=w_{j}^{n}e^{-\lambda}.
\end{equation*}

An important feature of our model is that, as can be seen above, information about individual $a_j^i$ and $v_j^i$ are not needed to update the values of $x_j$ and $w_j$. Given $z_n$, only $x_j^n$ and $w_j^n$ are needed to compute $x_j^{n+1}$ and $w_j^{n+1}$. 
Accordingly, 
we let $x^{n}=(x_{1}^{n},x_{2}^{n},\ldots,x_{k}^{n})$, where $x_{j}^{n}\in E$ represents the weighted mean of category $j$ at time step $n$ with associated weights $w^{n}=(w_{1}^{n},w_{2}^{n},\ldots,w_{k}^{n})$. These weights are the sum of all weights of the exemplars in each category.  We will refer to $x_{j}^{n}$ as the exemplar mean for category $j$.  The parameter $\lambda>0$ is our decay rate.

Initial conditions can be prescribed by giving the locations and weights of each exemplar in each category. However, since individual locations and weights of exemplars do not enter in to the dynamics of the Voronoi cells, we only need to set the category means and category weights.
   For each $j=1,2,\dots , k$, we define category means $x_{j}^{0} \in E$ such that $x_{j}^{0}\neq x_{i}^{0}$ whenever $j\neq i$, and let $w_{j}^{0} > 0$.



\Cref{fig:2dmodel_demonstration} shows what typical runs of the model look like with $\lambda=0$ and $\lambda>0$. In each case the model is initialized with 4 categories. The 4 categories were generated by selecting 4 arbitrary points in the phonetic space (in this case a 100 by 100 square) and distributing $100$ exemplars  with weight 1 about the points according to a  Gaussian distribution with standard deviation $3$. In both cases, at every time step a new exemplar was randomly generated uniformly on the square and assigned to the category of the category mean it is closest to.  In the $\lambda=0$ case exemplars remain at their initial weight for all time, and so the number of exemplars in the system and visible in the plot increases continually. The division of the square into a Voronoi tesselation converges to a centroidal Voronoi tesselation, as shown in \cite{macqueen}. In the $\lambda>0$ case on the right, although added exemplars remain in the system for all time, the weights of the exemplars decrease over time. To help visualize this, we only plot exemplar with weight greater than $0.01$. Accordingly, the number of exemplars in the plot converges to a steady state. However, the Voronoi cells now fluctuate for all time, as we show in the next section.

%

\section{General Results for $k$-Means Exemplar Model}
\label{sec:thm1}

In this section, we prove for the model described in Section~\ref{secdef} that none of the exemplar means converge in probability, and that the categories do not collapse.
In particular, our main result is the following.

\begin{thm}\label{thm:bigBeginningSection}
Let $E \subseteq \mathbb{R}^N$   be  a bounded, convex, closed subset having non-empty interior. Let $\lambda>0$ and 
$f(x)>0$ for all $x \in E$. Let initial conditions be $w^0_1, \ldots, w^0_k  > 0$, and   $x^0_1, \ldots x^0_k \in E$,  where $x^0_i \neq x^0_j$ for $i\neq j$. Let $z_n, n\geq 1$ be an independent sequence of random variables each with distribution given by $P$ as defined in \eqref{eqn:Pdef}.   Define $x_j^n, w_j^n$ for $n >0$ by: if $i$ is the minimal index such that $|x_i^n - z_n|$ is minimized, then
\begin{equation} \label{eqn:exempincat}
x_{i}^{n+1}=\dfrac{x_{i}^{n} w_{i}^{n} e^{-\lambda}+z_{n}}{w_{i}^{n}e^{-\lambda }+1},   \ \ \ \    w_{i}^{n+1}=w_{i}^{n}e^{-\lambda }+1,
\end{equation}
and for all $j\neq i$ we set 
\begin{equation} \label{eqn:exempoutcat}
x_{j}^{n+1}=x_{j}^{n},   \ \ \ \ \  w_{j}^{n+1}=w_{j}^{n}e^{-\lambda}.
\end{equation}
Then, for each $j$,
\begin{enumerate}
\item $x^n_j$ does not converge in probability to any random variable $x_j$  
as $n \rightarrow \infty$,
\item the volume of  $S^n_j$ (defined in \eqref{eqn:voronoidef}) does not converge to zero in probability as $n \rightarrow \infty$,
\item almost surely, $z_n \in S^{n}_j$ for infinitely many $n$.
\end{enumerate}
\end{thm}
\begin{proof} Result 1 is Corollary~\ref{cor:final} and Result 2 is implied by Theorem~\ref{areathm} and Lemma~\ref{lem:handy}, both of which are proved below. Result 3 follows from Result 1, since not converging weakly implies not converging almost surely, and the only way $x^n_j$ can move is if $z_n \in S^n_j$.
\end{proof}


Note that not converging in probability is a stronger condition than not converging almost surely or in mean. We conjecture that for each $j$ there \emph{is} convergence in distribution of $x_j^n$ to some random variable as $n \rightarrow \infty$, but we do not tackle this issue here.

The key to proving all three results is showing that for all sufficiently large $n$ with probability bounded away from zero, all exemplar means $x_j^n$ are separated from each other and from the boundary of $E$. This is an immediate consequence of Theorem~\ref{areathm} below.



To prove the results, we will first require some preliminary results, beginning with the following lemma.

\begin{lemma} There exists a $\gamma \in \R$ depending only on $\lambda$ and the initial vector of weights $w^{0}$, such that $w_{i}^{n}\leq \gamma$, for $i=1,2,\dots , k$, and for all $n\geq 0$.
\label{lem:weightbound}
\end{lemma}
\begin{proof}
If we let $W^{n}$ represent the total weight of our system at time step $n$, it is straightforward to show $W^{n}=\sum_{i=1}^{k}w_{i}^{n}=W^{n-1}e^{-\lambda}+1$ for all realizations. 
Since $e^{-\lambda} <1$ we know that $W^n$ converges to $W:= (1-e^{-\lambda})^{-1}$.
Since $W^{n}$ converges monotonically to W,
\beqa
W^{n}\leq \max \squigbr{W^{0}, \dfrac{1}{1-e^{-\lambda}} }=\gamma,
\label{weightbound}
\eeqa
for all $n$.
This in turn implies the result.
\end{proof}

Lemma \ref{lem:weightbound} proves the total weight of all exemplars in the system is uniformly bounded above. This is in contrast to the MacQueen model where the total weight of the system diverges, and new exemplars have less influence every iteration.
The value $W$ (that $W^{n}$ converges to) will come up later when we investigate the long term behaviour of the perceptual boundary.

The following lemma shows that for a fixed $r$ the probability of a new exemplar $z_n$ landing in a ball of radius $r$ centred at a point $x \in E$ is bounded away from 0, uniformly with respect to $x$ in bounded sets.

\begin{lemma}
If $r>0$ is fixed, and  closed $F\subseteq E$  then 
\beq
\inf_{x\in F}P(B(x,r) \cap E)>0.
\eeq
\label{nonempty_ball2}
\end{lemma}

\begin{proof}
We first want to show if $r>0$ is fixed, there exists an $r'>0$ such that for all $x\in F$ one can find a $x' \in E$ such that $B(x',r')\subseteq B(x,r)\cap E$.

Let $B_{0}=B(x_0,r_0)$ be a subset of $E$, where $r_0 >0$.  We know $B_{0}$ exists because $E$ has a non-empty interior.  Fix an $x$ in $F$. For any $\alpha \in [0,1)$, the set ${B_{\alpha}:=(1-\alpha)B_{0}+\alpha x}$ is an open ball which is contained in $E$ because  $E$ is convex.  Furthermore, $B_{\alpha}$ is centred at point $x_0+\alpha (x-x_0)$ and has radius $(1-\alpha)r_0$.

We want to find all such $\alpha$ so that $B_\alpha$ is completely within $B(x,r)$. This containment will hold  for $\alpha \in [0,1)$ if and only if
\[
(1-\alpha) |x - x_0| + (1-\alpha) r_0 \leq r
\]
or, rearranging,
\[
\alpha \geq 1 - \frac{r}{|x-x_0| + r_0}.
\]
Let $\alpha' := \max \left( 0, \sup_{x \in F} 1 - \frac{r}{|x-x_0| + r_0}\right)$. Since $F$ is bounded, $\alpha'<1$. Then $B_{\alpha'} = (1-\alpha')B_0 + \alpha' x$ is contained in $B(x,r)$ for all $x \in F$.  But $B_{\alpha'}$ is a ball of radius $r' = (1-\alpha') r_0$ and so we proved the existence of such an $r'$.

We have now shown that 
\beq
P(B(x,r)\cap E)\geq \int_{B(x,r)\cap E}f(y) dy \geq \int_{B(x'(x),r')}f(y) dy,
\eeq
where $x'(x) = x_0 + \alpha' (x- x_0)$.
To establish our result, we just need to show that
\beq
\inf_{x \in F}\int_{B(x'(x),r')}f(y) dy>0. 
\eeq

Suppose for contradiction that ${\inf_{x \in F}\int_{B(x'(x),r')}f(y) dy=0}$.  There must exist a sequence $\{ z_{n} \}_{n>0}$ in $F$ such that $\lim_{n\rightarrow \infty}\int_{B(x'(z_{n}),r')}f(y) dy=0$.
 Because $F$ is bounded, there exists a subsequence such that $z_{n_{i}}\rightarrow z\in E$, as $i\rightarrow\infty$. 
 Furthermore, since $B(x'(z_{n_i}),r') \subseteq E$, and $E$ is closed, $B(x'(z),r')$ is also a subset of $E$, where we have used the fact that $x'$ is a continuous function of $x$.
  We know that for all $\eps>0$, there exists an $I$, such that for all $i\geq I$, $|z_{n_{i}}- z|<\eps$.  Let $\eps=r'/2$, implying that $B(x'(z),\eps)\subseteq B(x'(z_{n_{i}}),r')$ for all $i\geq I$.  
 So for all $i \geq I$ we have 
\beq
\int_{B(x'(z_{n_{i}}),r')}f(x) dx\geq \int_{B(x'(z),\eps)}f(x) dx>0
\eeq
since $f$ is non-zero inside $E$, and $B(x'(z),\eps)$ is contained in $E$.
This gives us our contradiction, since we know the left hand side converges to zero as $i \rightarrow \infty$.
\end{proof}

The following lemma shows that if a long enough sequence of new exemplars arrive within $\eps$ of a point, then some category mean will arrive within $2 \eps$ of the point.


\begin{lemma} 
Let $z \in E$ and $\eps>0$ be given. There exists a $p>0$, such that if  $z_{q}$ is in $B(z,\eps)$ for $n\leq q \leq n+p-1$, then for some  $m \in \{ 1, \hdots , k \}$,  the exemplar mean $x_{m}^{n+p}\in B(z, 2\eps)\cap E$.  Parameter $p$ only depends on $k$,
$\mathrm{diam}(E)$ (the diameter of set $E$) and 
$\eps$.
\label{balllemma}
\end{lemma}

\begin{proof}

Let's assume that $z_{q}$ is in $B(z,\eps)$ for $q \in \{n, \ldots, n+kp'-1\}$, for some $p'$ we will determine later.
We know that one category $m \in \{ 1, \hdots , k \}$ will be classified at least $p'$ times over that time interval.  So there exists a subsequence $\{ q_{i} \}_{i\geq 1} \subseteq \{ n, \hdots , n+ kp'-1 \} $ such that $z_{q_{i}}\in S_{m}^{q_{i}}$ for all $i$, and $| \{ q_{i} \}_{i\geq 1}|\geq p'$.  Note that in this time interval 
the exemplar means either stay fixed or move closer to $z$, if they are not already in the ball $B(z,\eps)$.
 We want to show there is a $p$ large enough such that we are guaranteed the $m$th exemplar mean will be inside $B(z,2 \eps)$ by the time $n+kp'$.

Let $\eta:= \min_{i } |x_i^n - z|$ be the distance of the closest exemplar mean to $z$ at time $n$.  Note that since new exemplars only arrive at locations in $B(z,\eps)$, only exemplars within $\eta + \eps$ of $z$ will move. In particular, $|x_m^{n} - z| \leq \eta + \eps$.

Without loss of generality, let us assume $z$ is at $0$, so that $|z_q| \leq \eps$ for all $q$. 
Let $y^q: = |x_{m}^q|$.

If $z_{q}\in S_{m}^{q}$, then
\beq
x_m^{q+1} =\dfrac{x_m^{q} w_{m}^{q} e^{-\lambda}+z_{q}}{w_{m}^{q}e^{-\lambda}+1}.
\eeq
Let  $\rho =  w_{m}^{q} e^{-\lambda}/(w_{m}^{q}e^{-\lambda}+1)$. Note that $\rho \in (0,1)$ and is bounded away from 1, because of the bound on the total weights provided by Lemma \ref{lem:weightbound}. So
\[
x_m^{q+1} = \rho x_m^{q} +  (1- \rho) z_{q}
\]
which implies
\begin{eqnarray*}
y^{q+1} & \leq & \rho  y^{q} + (1- \rho)|z_{q}| \\
& \leq &\rho y^q + (1-\rho) \eps.
\end{eqnarray*}
If $z_q \not\in S_{m}^{q}$ then $y^{q+1}=y^q$.  

We know that new exemplars will be classified in category $m$ at least $p'$ times.
So, using the inequality version of the identity for geometric series:
\[
y^{n+kp'} \leq \rho^{p'} y^n +  \frac{1- \rho^{p'}}{1 - \rho} (1-\rho)  \eps \leq  \rho^{p'} (\eta + \eps) + (1- \rho^{p'}) \eps \leq \rho^{p'} \mbox{diam}(E) + \eps.
\]
The limit  of the right-hand side as $p' \rightarrow \infty$ is $\eps$, so there is a large enough $p'$ so that $y^{n+kp} < 2 \eps$, and hence $x_{m}^{n+kp'} \in B(z, 2 \eps)$. Let $p = k p'$ gives the required result. 
\end{proof}

Lemma \ref{balllemma} will be used to prove that the exemplar means don't converge towards one another, and as such, there is no collapse in the system.  How it will be utilized will become apparent in the following lemma. Here we show if a collection of one or more exemplar means are close to a point $z$ in the interior of $E$, with positive probability  one of the exemplar means will be moved away from $z$ in a bounded number of steps. Meanwhile, all the exemplar means that are far away from $z$ will not be moved.

In what follows, we will use $\partial E$ to denote the boundary of $E$, and for any subset $F$ of $\mathbb{R}^N$, $d(z,F)$ to denote the distance between point $z \in \mathbb{R}^N$ and $F$:
\[
d(z,F) = \inf_{x \in F} |z-x|.
\]

\begin{lemma} 
Let $\delta>0$, $z \in E$ with $\delta \leq d(z,\partial E)$, $\eps$ be a constant such that $\eps \leq \delta/10$, and
\begin{eqnarray*}
A_{1}&=&\{ i, \mbox{ s.t. }|x_{i}^{n_{0}}-z|\geq \delta \}, \\
A_{2}&=&\{ i, \mbox{ s.t. }|x_{i}^{n_{0}}-z| < \delta/2-4\eps, \mbox{ and }x_{i}^{n_{0}}\neq z\},  \\
A_{3}&=&\{ i, \mbox{ s.t. }x_{i}^{n_{0}}=z \}.
\end{eqnarray*}
If $|A_{1}|+|A_{2}|+|A_{3}|=k$ (so there are no exemplars between distances $\delta/2-4\eps$ and $\delta$ from point $z$), then there exists a $y\in B(z,\delta/2)$, and a $p>0$, such that if $z_{n}\in B(y,\eps)$ for $n_{0}\leq n < n_{0}+p$, then $\max_{i\in A_{2}}|x_{i}^{n_{0}+p}-z|\geq 
\delta/2 -4\eps$, and $x_{i}^{n_{0}+p}=x_{i}^{n_{0}}$ for all $i\in A_{1}\cup A_{3}$.
\label{backgthm1}
\end{lemma}
\begin{proof}
Let $q=\argmax_{i\in A_{2}}|x_{i}^{n_{0}}-z|$, and
\beq
y=z+\dfrac{x_{q}^{n_{0}}-z}{|x_{q}^{n_{0}}-z|}(\delta/2-2\eps),
\eeq
so that $y$ is a distance $\delta/2-2\eps$ away from $z$ in the direction of point $x_{q}^{n_{0}}$.  We know because $\delta \leq d(z,\partial E)$, that $B(y,\eps)\subseteq E$.  If $z_{n}\in B(y,\eps)$, consecutively, then the $n$th sound will always be categorized as a category in $A_{2}$ for the following reasons:
\begin{enumerate}
\item $x_q^{n}$ is closer to any point in $B(y,\eps)$ than $z$ is. So new exemplars are never classified in categories $i \in A_3$.

\item $x_{q}^{n}$ will always be closer to $B(y,\eps)$ than any $x_{i}^{n}$ such that $i\in A_{1}$. So new examplars are never classified in categories $i \in A_1$.
\end{enumerate}

Let constant $p>0$ be as determined by Lemma \ref{balllemma}, which will depend on $k$, $\eps$ and $\mbox{diam}(E)$.  If $z_{n}\in B(y,\eps)$ for $n_{0}\leq n\leq n_{0}+p$, we know $z_{n}\in S_{i}^{n}$ where $i\in A_{2}$ for all $n$ such that $n_{0}\leq n\leq n_{0}+p$.  By Lemma \ref{balllemma}, we know there exists a category $m$ such that $x_{m}^{n_{0}+p}\in B(y,2\eps)$.  The category $m$ must be in $A_{2}$ because the categories in $A_{1}\cup A_{3}$ do not move in the time interval.  This implies there exists an $i\in A_{2}$, such that $x_{i}^{n_{0}+p}\in B(y,2\eps)$, and because $|y-z|=\delta/2-2\eps$, we know that $|x_{i}^{n_{0}+p}-z|\geq \delta/2-4\eps$, giving the result.
\end{proof}

Using Lemma \ref{backgthm1} we will establish Theorem~\ref{awayexempthm} which states the following:  for a given $j$, as long as all exemplar means are away from the boundary of $E$, there is a probability bounded away from zero that at some time later,  all other exemplar means will be moved away from the $j$th one and the $j$th one will be moved away from the boundary.


\begin{thm} 
For any $\delta>0$, $j\in {1, \ldots,k}$, and time $n_0$,  there exists an $\eps>0$, $M>0$, and $H>0$ such that 
\beq
\mathbf{P}\brac{  \min_{i\neq j}|x_{i}^{n_0 + M}-x_{j}^{n_0 + M}| \geq \eps, d(x_j^{n_0 + M},\partial E) \geq \eps  \bigg|
d(x_j^{n_0}, \partial E) 
 \geq \delta }\geq H
\eeq
In particular, $\eps$, $M$, and $H$ depend only on $\delta$, $E$, and $k$.
\label{awayexempthm}
\end{thm}

\begin{proof}
To prove this theorem, we will begin by showing there is an event which pulls all the exemplar means (except category $j$'s) away from category $j$'s exemplar mean $x_{j}^{n}$,
under the condition $d(x_{j}^{n_{0}},\partial E)\geq \delta$. Futhermore, $x_j^n$ will also be away from the boundary.
  We then will show the event has a positive probability of happening. 

Before describing the event, we will define some parameters.  We let $\eps =\delta 2^{-k}/9$.  
  Let $\{d_{i} \}_{i=1}^{k-1}$ be defined as $d_0=\delta$  and for all $i>1$, $d_{i+1}=d_{i}/2 -4\eps$.  One can calculate $d_{i}$ explicitly as
\begin{equation*}
d_{i} 
=\dfrac{\delta}{2^{i}}-8\eps\brac{1-\dfrac{1}{2^{i}}}
\end{equation*}
for $i\in\{1,\hdots,k-1 \}$.  This is a decreasing sequence, such that $B(x_{j}^{n_{0}},d_{i}+4\eps)\subseteq E$ for $i\geq 1$.
Our choice of $\eps$ guarantees that $d_{i}\geq \eps$ for all $i$.
%

An event will now be described which pulls every exemplar mean (except for category $j$'s) at least distance $\eps$ away from $x_{j}^{n}$.
The event will be described in an algorithmic manner, as a sequence of steps that must occur as the index $i$ runs from $1$ to $k-1$. For each step, what needs to happen depends on whether or not there are already $i$ exemplar means distance $d_i$ or greater from $x_j^n$. If there are not, we use Lemma~\ref{backgthm1} to pull one of the exemplars within distance $d_i$ from $x_j^n$ away from it. If there already are, we select new exemplars from a ball of radius $\eps$ far away from $x_j^n$ so as not to disturb exemplars close to $x_j^n$.
At the end of  step $i$, we will be guaranteed at least $i$ exemplar means will be distance $d_i$ or farther from $x_j^n$.

Let $z=x_{j}^{n_{0}}$ and $n=n_0$. We will start our process with $i=1$ and proceed through to $i=k-1$.


For $i = 1$ to $k-1$, 
let $s_i^n = | \{ q \mbox{ s.t. } |x_q^n-z| \geq d_i \}|$, the number of category means that are distance $d_i$ or greater from $z$. 
There are two possibilities. 

{\bf Case 1:} $s_i^n = i-1$.

We need to move at least one exemplar mean inside $B(z,d_i)$ outside of $B(z,d_i)$, without moving $x_j^n$ or any of the exemplar means that are outside of $B(z,d_{i-1})$.

By the inductive hypothesis, we know that there are  $i-1$ exemplars farther than $d_{i-1}$ from $z$. So this means that there are no exemplars in $B(z,d_{i-1})\backslash B(z,d_{i})$.

  We also know $z$ is at least distance $\delta$ from the boundary, $d_{i}=d_{i-1}/2-4\eps\geq \eps$, and $\delta>d_{i}$, which implies that we can implement Lemma \ref{backgthm1}.

By Lemma \ref{backgthm1}, there exists a $y\in B(z,d_{i}+4\eps)$, and a $p>0$, such that if $z_{m}\in B(y,\eps)\subseteq E$, for $n\leq m \leq n+p-1$, then $\max_{i\in A_{2}}|x_{i}^{n+p}-z|\geq d_{i}$, and $x_{i}^{n+p}=x_{i}^{n}$ for all exemplar means outside of $B(z,d_{i-1})$.  In other words, if new sounds are generated consecutively in $B(y,\eps)$, one exemplar mean will be moved outside of $B(z,d_i)$, but neither $x_j^n$ nor the exemplar means outside of $B(z,d_{i-1})$ will be moved.
We can take the $p$ determined by the lemma, which only depends on $k$, $\mbox{diam}(E)$, and $\eps$, and not on the specific configuration of category means.


{\bf Case 2:} $s_i^n \geq i$. There are already enough exemplars far enough away from $z$ at this step. 
Let $\ell$ be such that the exemplar $x_\ell^n$ is furthest from $z$. Let $y$ be distance $\delta-\eps$ away from $z$ in the direction of $x_\ell^n$. Allow the next $p$ exemplars introduced to the system to fall within $B(y,\eps)$, where $p$ is the same value that would have been chosen in Case 1. The effect of this will only be to move around exemplar means of distance farther than $d_i$ from $z$, and it cannot move them within $d_i$ of $z$.

Now observe that whichever of these cases occur, the exemplar mean $x_j^n$ is not moved. Since it started out at least distance $\delta$ from the boundary (and $\eps < \delta$), we complete this step with $x_j^n$ at least $\eps$ away from the boundary.

After one of these two cases has been performed for $i = 1$ to $k-1$,
the sequence of $p(k-1)$ events guarantees that all exemplar means except $x_j^n$ are distance $\eps$ away from $x_j^n$ at time $n= n_0 + p(k-1)$, and $x_j^n$ is also at least distance $\eps$ from the boundary. We just need to show that the probability of this event is bounded away from $0$.
By Lemma~\ref{nonempty_ball2}, letting $F=E$, each event ($z_n \in B(y,\eps)$) has a probability greater than $Q>0$ uniformly with respect to $y$. So the total event has probability at least $Q^{p(k-1)}$ as required.   Letting $H=Q^{p(k-1)}$ and $M=p(k-1)$ gives the result.
\end{proof}

The probability in Theorem \ref{awayexempthm} is conditioned on all the category means $x_j^{n_0}$ being a distance $\delta$ away from $\partial E$.
We have to establish that this event happens with a non-zero probability.  We will show there exists a sequence of events which bring all the exemplar means at least a distance $\delta^{*}$ away from $\partial E$.  The probability of the event occurring is bounded below like in Theorem~\ref{awayexempthm}.  Two lemmas are required to prove it.

\begin{lemma}
Let $z^{*}$ be in the interior of $E$, $\Gamma \in (0,1)$, and define the function $f(x)=z^{*}+\Gamma (x-z^{*})$.
There exists a $\delta>0$  such that $d(f(x),\partial E)\geq \delta$, for all $x\in E$.
\label{boundarylemma}
\end{lemma}

\begin{proof}
We first observe that for all $x \in E$, $f(x)$ is in the interior of $E$. To see this, let $x$ be in the interior of $E$. Note that $B(z^*, \rho)$ must be contained in $E$ for some $\rho>0$ because $z^*$ is in the interior of $E$. By convexity, $B(z^* + \Gamma (x- z^*), (1-\Gamma) \rho)$ must also be contained in $E$. So $f(x)$ is in the interior of $E$.

Since $E$ is compact, $f$ is continuous, and continuous images of compact sets are compact, $f( E)$ is a compact set. Two compact sets must have respective points whose distance is the same as the distance between the sets. Since nowhere do $f(E)$ and $\partial E$ intersect, this distance must be positive. So, there exists a $\delta>0$ such that $\mathrm{dist} ( f( E), \partial E) \geq  \delta$.
\end{proof}


Before proving the next necessary lemma, we must define a sequence. Let $\Gamma = 1/2$.   For $\eps>0$, let $\{ f_{i} \}_{i=1}^{k}$ be a sequence such that $f_{1}=0$ and 
\beq
f_{i+1}=\dfrac{\Gamma+(f_{i}+2\eps)}{2}+2\eps,
\eeq
for all $i>1$.  
$f_{i}$ can be written explicitly as $f_{i}=(\Gamma + 6 \eps) (1-2^{1-i})$.   
We will need to choose  $\eps$  small enough so that $f_i<\Gamma$ for $i=1,\ldots, k$. This is accomplished by letting $\eps \leq 2^{-k}/6$. 
The sequence $\{ f_{i} \}_{i=1}^{k}$ will play a  similar role to the sequence $\{ d_{i} \}_{i=1}^{k-1}$ in Theorem \ref{awayexempthm}, though, $\{ f_{i} \}_{i=1}^{k}$ does not depend on the set $E$, ranges between $0$ and $1$, and increases instead of decreases.

\begin{lemma} 
Fix a time $n_0$.
Let $z^{*}$ be in the interior of  $E$ and let $\alpha >0$ be such that $B(z^*,\alpha) \subseteq E$.
Let $\{ c_{q} \}_{q=1}^{k}$ be a reordering of $\{ 1, \hdots , k \}$
 such that
\beq
|x_{c_{1}}^{n_{0}}-z^{*}|\leq |x_{c_{2}}^{n_{0}}-z^{*}|\leq \hdots \leq |x_{c_{k}}^{n_{0}}-z^{*}|.
\eeq
Let $q\in \{ 1, \hdots , k \}$, and let $s = \max\{ \alpha,  |x_{c_{q}}^{n_{0}}-z^{*}| \}$.
 Let $\Gamma = 1/2$ and let 
\[
\eps=\min \left\{  \dfrac{\delta}{ 3 \mathrm{diam}(E)}, \dfrac{2^{-k}}{6}  \right\}>0.
\]
Let  $f_i=(1-2^{1-i} )(\Gamma+6 \eps)$ for $i = 1,\ldots, k$, and 
  let $\delta>0$  be the constant determined by Lemma \ref{boundarylemma}, where we use $z^{*}$ as the interior point in $E$.

If $|x_{c_{q}}^{n_{0}}-z^{*}| > (f_{q}+2\eps)s$, and (if $q>1$) $ |x_{c_{q-1}}^{n_{0}}-z^{*}| \leq (f_{q-1}+2\eps)s$, 
then there exists a $y$ in the interior of $E$, and a $p>0$, such that if $z_{n}\in B(y,\eps \alpha)$, for $n_{0}\leq n < n_{0}+p$, there will exist an $i\geq q$ such that
\beq
|x_{c_{i}}^{n_{0}+p}-z^{*}|\leq (f_{q}+2\eps)s,\mbox{ and }
d(x_{c_{i}}^{n_{0}+p},\partial E)
\geq \delta/3.
\eeq
Additionally $x_{c_{r}}^{n}=x_{c_{r}}^{n_{0}}$, for all $r<q$, and $n$ such that $n_{0}\leq n \leq n_{0}+p$.
\label{backgthm2}
\end{lemma}

\begin{proof}
Let 
\[
y=z^{*}+ \dfrac{ (x_{c_q}^{n_0} - z^*)} { | x_{c_q}^{n_0} - z^* |} s f_q,
\]
so $y$ is distance $s f_{q}$ away from $z^{*}$ in the direction of point $x_{c_{q}}^{n_{0}}$.  
Note that if we let $w = z^* + \frac{s f_q}{\Gamma} \frac{(x_{c_q}^{n_0} - z^*)}{| x_{c_q}^{n_0} - z^* |}$, then 
\begin{equation}
y= z^* + \Gamma (w - z^*),
\end{equation}
where $w \in E$, since it is a convex combination of $z^*$ and $x_{c_q}^{n_0}$.
By Lemma \ref{boundarylemma}, we therefore have that $y$ is at least distance $\delta$ from $\partial E$.

We will show that exemplars $z_n$ falling in $B(y, \eps \alpha)$ will always be classified in categories $c_i$ for $i\geq q$. First note that this is immediate if $k=1$ so assume $k \geq 2$.  For $z_n \in B(y, \eps \alpha)$ we have
\begin{eqnarray*}
|z_n - x_{c_q}^{n_0}|  & \leq &  |y - x_{c_q}^{n_0}| + |y - z_n| \\
& \leq & s- f_q s + \eps \alpha \\
& \leq & s [1 - f_q + \eps],
\end{eqnarray*}
and, for $i < q$,
\begin{eqnarray*}
|z_n - x_{c_i}^{n_0}| & \geq & |y - z^*| - |z^* - x_{c_i}^{n_0}| - |z_n-y| \\
& \geq & f_q s - (f_{q-1} + 2 \eps) s - \eps \alpha \\
& \geq & s \left[ f_q - f_{q-1} - 3 \eps \right].
\end{eqnarray*}
The definition of $f_q$ in terms of $f_{q-1}$ then allows us to show $|z_n-x_{c_q}^{n_0}| < |z_n,x_{c_i}^{n_0}|$, and so exemplars falling in $B(y, \eps \alpha)$ will always be classified in categories $c_i$ for $i\geq q$.

Let $p$ be the constant determined by Lemma \ref{balllemma} which will depend on $\mbox{diam}(E)$ and $\eps \alpha$ (in place of $\eps$ in the lemma).  If $z_{n}\in B(y,\eps \alpha)$ for all $n$ such that $n_{0} \leq n < n_{0}+p$, there must exist an $m$ such that $x_{m}^{n_{0}+p}\in B(y,2\eps \alpha)$.  We know $m\geq q$, because the other exemplar means cannot move.  
As such there exists an $m\geq q$, such that $x_{c_{m}}^{n_{0}+p}\in B(y,2\eps s)$, since $s \geq \alpha$.  This gives 
\[
|x_{c_m}^{n_0+p}- z^*| \leq |x_{c_m}^{n_0+p}- y| + | y- z^*| \leq 2 \eps s + f_q s = (f_q + 2 \eps) s.
\]
Additionally $x_{c_{r}}^{n_{0}+p}=x_{c_{r}}^{n_{0}}$, for all $r<q$.

We know $\eps\leq \delta/ \brac{3\mbox{diam}(E)}$, which ensures that $\eps s\leq \delta/3$.   Since $y$ is at  least distance $\delta$ from $\partial E$ and the radius of the ball $B(y, 2\eps s)$ is at most $2 \delta/3$, we have that
$d(x_i^{n_0+p},\partial E)\geq \delta/3$.
\end{proof}

Lemma \ref{boundarylemma} will be used in the proof of the following theorem.  The proof of the theorem will be similar to the proof of Theorem \ref{awayexempthm}; we will describe an event in which all exemplar means are pulled away from the boundary, and which has a positive probability of occurring.

\begin{thm} 
There exists a $\delta^{*}>0$, an $M>0$, and an $H>0$ such that, for every time $n_0$
\beq
\mathbf{P}\brac{ \min_{i} d(x_{i}^{n_0 + M},\partial E ) \geq \delta^{*} }\geq H.
\eeq
$\delta^*$, $M$, and $H$ only depend on $E$.
\label{awayboundarythm}
\end{thm}

\begin{proof}
To prove this lemma, we are going to show there is an event which can bring all the exemplar means away from $\partial E$.  We will prove this event has a positive probability of happening.  Before describing the event, we will define some variables.

Let $z^{*}$ be a point in the interior of $E$.  Let  $\alpha=d(z^{*},\partial E)>0$. 

We use the sequence $\{ f_{i} \}_{i=1}^{k}$, which was defined before Lemma \ref{backgthm2}.
To repeat: let $\Gamma=1/2$, and use Lemma~\ref{boundarylemma} to give us a $\delta>0$ so that $d(z^* + \Gamma(x-z^*),\partial E) \geq \delta$ for all $x \in E$.  
Let  $\eps = \min \left\{ \delta/(3 \mbox{diam}(E)), 2^{-k}/6 \right\}$.
Then let
 $f_{i}=(1-2^{1-i})(\Gamma+6\eps)$ for $i=1,\ldots, k$.


As in Theorem~\ref{awayexempthm} we will describe a sequence of steps that pulls the $i$th exemplar to within distance $(f_i + 2\eps) \mbox{diam}(E)$ of $z^*$ and where all the exemplars are at least distance $\delta^*  = \delta/3$ away from $\partial E$. We start with $i = 1$ and then increase $i$ up to $k$.

Let $n=n_0$.
We repeat the following steps for $i =1, \ldots, k$.
First, we let the indices $c_\ell, \ell=1,\ldots,k$ be such that 
\[
| x_{c_1}^{n} - z^*| \leq | x_{c_2}^{n} - z^* | \leq \cdots \leq | x_{c_k}^{n} - z^* |
\]
Let  $s = \max\{ \alpha, |x_{c_i}^n - z^*| \}$.
Let $s_i^n = | \{ q \mbox{ s.t.} | x_q^n - z^* | < (f_i + 2 \eps)s \} | $ be the number of category means that  are within $(f_i + 2 \eps) s$ of $z^*$. 
There are two possibilities.

{\bf Case 1:} $s_i^n = i-1$.

We need to move at least one exemplar outside of $B(z^*, (f_i + 2 \eps)s)$ to within distance  $(f_i + 2 \eps)s$ of $z^*$. Lemma~\ref{backgthm2} shows precisely this: there is a $p$ such that if $z_m$ falls in $B(y,\eps \alpha)$ for $n \leq m< n + p$ then at least one exemplar mean will be moved in to $B(z^*, (f_i + 2 \eps)s)$. Furthermore, none of the exemplar means that are already closer to $z^*$ will be moved, and the exemplar mean that is moved will be farther than distance $\delta/3$ of $\partial E$.

{\bf Case 2:} $s_i^n \geq i$.

In this case, none of the exemplar means need to be moved for the condition to be satisfied. In this case we allow $p$ exemplars in a row to fall within $B(z^*,\eps \alpha)$.  Since $\eps \alpha \leq (f_i + 2 \eps) s$ for  all $i$, we have $B(z^*,\eps \alpha) \subset B( z^*, (f_i + 2 \eps) s)$  for all $i$, and this does not change any of the necessary containments. 

Finally, we update $n$ to $n+ p$.

Repeating the procedure for $i =1,\ldots, k$ gives the required conditions. What was necessary was that event of the form $z_n \in B(y,\eps \alpha)$ for some $y\in E$,  $k p$ times in a row. Lemma~\ref{nonempty_ball2} shows that the probability of $z_n \in B(y,\eps \alpha)$ is bounded below uniformly in $y \in E$ by some $h >0$, and so the probability of it happening $kp$ times in a row for varying $y$ is bounded below by $H = h^M$, where $M=kp$.
\end{proof}


 With these theorems, we now have enough to prove there is no categorical collapse for the system described in Section~\ref{secdef}.  

We introduce two new definitions. First, let $\mathcal{F}_n$ for $n \geq 0$ be the $\sigma$-algebra generated by $\{ x_j^m, w_j^m, z_{m-1} \}_{m \leq n}$, so that $\mathcal{F}_n$ represents everything that has happened up to and including time $n$. Let $C_\eps$ be the set of all  $x \in E^k$ such that all $x_j$ are at least distance $\eps$ away from each other and from the boundary of $E$:
\[
C_\eps = \left\{ x \in E^k \bigg| \min_{i \not = j} |x_i - x_j| > \eps, \min_i d(x_i, \partial E) > \eps \right\}.
\]
We now establish that there is an $\eps>0$ such that at any time $n$, conditioning on the present state there is a probability bounded away from zero that the process will enter $C_\eps$ after a fixed number of steps.

\begin{thm}
There exists an $\eps>0$, an $H>0$, and an integer $M>0$ such that, for each category $j$ and each time $n$,
\[
\mathbf{P}( x^{n+M} \in C_\eps | \mathcal{F}_n ) > H.
\]
Consequently, $\mathbf{P}(x^n \in C_\eps ) > H$ for all $n \geq M$.
\label{areathm}
\end{thm}
\begin{proof}
Let time $n \geq 0$
 be given.
By Theorem \ref{awayboundarythm}, we know there exists a $\delta^{*}>0$, $M_{1}>0$, and $H_{1}>0$, such that
\beq
\mathbf{P}\brac{ \min_i d(x_{i}^{n+M_1}, \partial E) \geq \delta^{*}  |  \mathcal{F}_n}\geq H_{1}.
\eeq
By Theorem \ref{awayexempthm} there exists an $\eps_{1}>0$, $M_{2}>0$, and an $H_{2}>0$, such that
\beq
\mathbf{P}\brac{  \min_{i\neq j}|x_{i}^{n+M_1 + M_2}-x_{j}^{n + M_1 + M_2}| \geq \eps_{1}, \min_i d (x_{i}^{n + M_1 + M_2},\partial E) \geq \eps_1 \bigg| d(x_{j}^{n+ M_{1}},\partial E) \geq \delta^{*} , \mathcal{F}_{n+M_1}}\geq H_{2}.
\eeq
Combining these two bounds we get that 
\[
\mathbf{P}\brac{  \min_{i\neq j}|x_{i}^{n+M_1 + M_2}-x_{j}^{n + M_1 + M_2}| \geq \eps_{1}, d(x_{j}^{n + M_1 + M_2},\partial E )   \geq \eps_1 \bigg| \mathcal{F}_n}\geq H_1 H_{2}.
\]
Letting $\eps= \eps_1/2$, $H=H_1 H_2$ and $M=M_1+M_2$ gives the result.
%
\end{proof}

Next we establish the consequences of a vector of exemplar means $x^n$ being in $C_\eps$: the volume of each Voronoi cell is bounded away from $0$ and each exemplar mean has a considerable chance of moving far.

\begin{lemma}
\label{lem:handy}
If $x^n \in C_\eps$ then  for all $j$
\begin{enumerate}
\item the volume of $S^n_j$ is greater than or equal to that of a sphere of radius $\eps/2$ in $\R^N$,
\item for some $H'>0$, $\mathbf{P}( |x^{n+1}_j - x_j^n | > \eps/4 \gamma | \mathcal{F}_n) > H'$,
\end{enumerate}
where $\gamma$ is defined in Lemma \ref{lem:weightbound} and $H'$ only depends on $E, f$ and $\eps$.
\end{lemma}
\begin{proof}
The first result follows from the definitions of $C_\eps$ and $S^n_j$. 

For the second result, fix a category $j$ and
let $F=B(x_{j}^{n}, \eps/2) \backslash B(x_{j}^{n}, \eps /4)$.  Because $x^n \in C_\eps$, $F$ lies entirely in $E$, and furthermore lies entirely in $S_j^n$.
We know there always exists a ball of radius $\eps /8$ which is a subset of $F$.  This implies by Lemma \ref{nonempty_ball2} that the probability of the event $\{ z_{n}\in F \}$ is bounded below by some constant $H'>0$.

If $z_{n}\in F$, then $|z_{n}-x_{j}^{n}| \geq \eps /4$, implying
\begin{align*}
|x_{j}^{n+1}-x_{j}^{n}| =\dfrac{|z_{n}-x_{j}^{n}| }{w_{j}^{n+1}} > \dfrac{\eps}{4 \gamma},
\end{align*}
as required.
\end{proof}


Putting together the bound on probability of being in $C_\eps$ in Theorem~\ref{areathm} with the consequences of being in $C_\eps$ from being in Lemma~\ref{lem:handy} gives us the following, which is also the Result 1 of Theorem~\ref{thm:bigBeginningSection}.

\begin{cor} \label{cor:final}
For all $j \in \{1,\ldots,k\}$ the exemplar mean $x_j^n$ does not converge in probability as $n \rightarrow \infty$.
\end{cor}
\begin{proof}
Suppose for some random variable $x_j$ taking values in $(\R^N)^k$, $x_j^n$ converges in probability to $x_j$.
So for all $\delta>0$, $\mathbf{P} (  |x_j - x_j^{n}|> \delta ) \rightarrow 0$ as $n \rightarrow \infty$. 

Let $\eps$, $H$ and $M$ be given as in Theorem~\ref{areathm}. Then using Lemma~\ref{lem:handy}, we have $n \geq M$ implies $\mathbf{P}( |x_j^{n+1} - x_j^n| > \eps/4 \gamma) > H H'$.
But
\[
\mathbf{P}( |x_j^{n+1} - x_j^{n}|> \eps/4 \gamma) \leq \mathbf{P}( |x_j^{n+1} - x_j|> \eps/8 \gamma)+ \mathbf{P}( |x_j - x_j^{n}|> \eps/8 \gamma).
\]
So the quantities on the right cannot both converge to $0$, which contradicts our assumption of convergence in probability.
\end{proof}

%
%
%
%
%
%
%
%
%

\section{A Simple Model for the Motion of the Perceptual Boundary}
\label{sec:twocbeh}

In the remainder of the paper will  study a simple special case of our model.  We consider a system with just two categories. We let our domain be  $E=[0,1]$ and  we stipulate that new exemplars arrive in the system with uniform probability density on $E$.  These choices correspond to $k=2$ and $f(x)=1$ for all $x \in E$.
Figure \ref{fig:1dex} shows a state of our model for these choices.
We perform a detailed study of the dynamics of the perceptual boundary in this case, providing a simple probabilistic model of its motion.

\begin{figure}[h!]
\begin{center}
	\includegraphics[width=13cm]{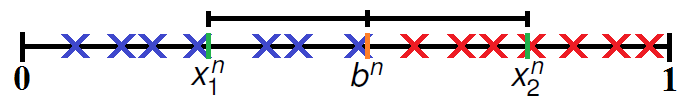}
	\end{center}
	\caption{A representative state of our model in the special case of $E=[0,1]$, uniform probability distribution on $[0,1]$ and two categories (blue and red). $x_1^n$ and $x_2^n$ are the corresponding category means, and $b^n$ is the perceptual boundary between the two categories.    }
	\label{fig:1dex}
\end{figure}

Our goal in this section is to characterize the behaviour of $b^n$, the location of the perceptual boundary as function of $n$.  The sequence $b^n, n\geq 0$ is a discrete-time stochastic process, where the randomness in the evolution of $b^n$ enters through the i.i.d.\ random variables $z_n, n\geq 0$. Even in this simplest case of our model, we are not able to get a complete analysis of the system, so we instead derive an even simpler approximation to our model.
We derive an autoregressive first-order ($AR(1)$) model as an approximation to the behaviour of $b^n$ \cite{madsen2007}. We find the best approximation to the dynamics of $b^n$ among $AR(1)$ models, deriving coefficients in terms of $\lambda$. Recall that an $AR(1)$ model for a real time series  $y_n \in \R$, $n=0,1,2,\ldots$ is given by 
\begin{equation}
y_{n+1} = k y_n  + \sigma \eta_n
\end{equation}
where $k \in \R$, $\eta_n$ is a sequence of i.i.d.\ standard Gaussian random variables and $\sigma \geq 0$ is a scalar.  When $|k|<1$ the sequence $y_n$ converges to a stationary stochastic process with stationary distribution $N(0,\sigma^2/(1-k^2))$.  $y_n$ fluctuates about $y=0$ with a limiting stationary autocovariance function given by \cite[p. 147]{madsen2007}
\begin{equation} \label{eqn:autocov}
C_{r} = \lim_{n \rightarrow \infty} \mathbb{E}  [ y_n y_{n+r} ] = \sigma^2 k^{|r|} /(1- k^2).
\end{equation}

We assume the initial exemplar means are ordered such that $x_{1}^{0}<x_{2}^{0}$, implying $x_{1}^{n}<x_{2}^{n}$, for all $ n\geq 0$.  Let the perceptual boundary between the $2$ exemplar means be $b^{n}$.  It is straightforward to show $b^{n}=(x_{1}^{n}+x_{2}^{n})/2$.


To motivate the use of an $AR(1)$ model for  the time series of $n$, see 
Figure \ref{fig:example} in which we show time series for $x^n_1, x^n_2$ and $b^n$.  The left graph shows the evolution of the system where $\lambda>0$, and the right where $\lambda=0$ (MacQueen's model \cite{macqueen}).  The red lines represent the two weighted exemplar means $x_{1}^{n}$ and $x_{2}^{n}$.  The blue line is the perceptual boundary given by $b^{n}=(x_{1}^{n}+x_{2}^{n})/2$, which represents the boundary between the Voronoi cells ($S_{1}^{n}$ and $S_{2}^{n}$) of the two categories.  
In the left plot ($\lambda>0$) the category means fluctuate with roughly the same amplitude for the whole interval, whereas on the right they appear to converge.


\begin{figure}[h!]
	\includegraphics[width=13cm]{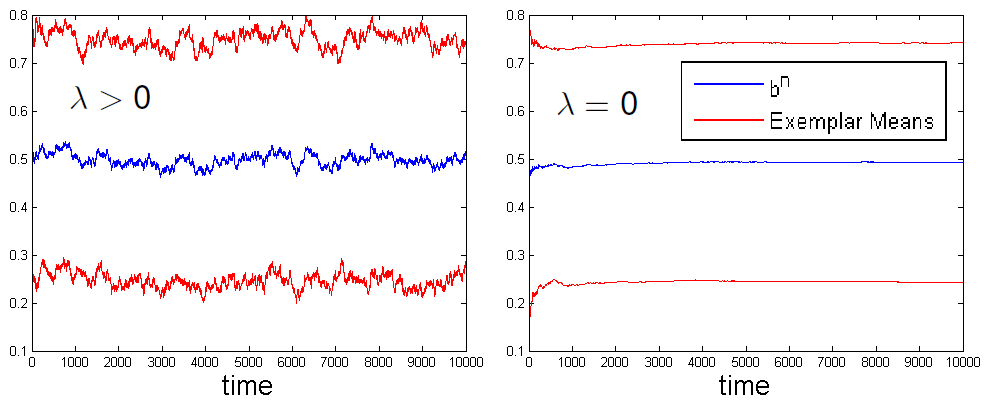}
	\caption{A comparison of the evolution of simulations of the two-category uniform distribution system $1$-D case.  The left graph exhibits what the system looks like when $\lambda=0.01$ and initial values $w_{1}^{0}=w_{2}^{0}=W/2 \cong 50$.  On the right we set $\lambda=0$ and $w_{1}^{0}=w_{2}^{0} = 10$.}  
	\label{fig:example}
\end{figure}

Recall from Section~\ref{secdef} that the evolution of the category means of the system can be completely expressed using only the category means and the category weights: the positions of individual stored exemplars do not enter the dynamics.
Accordingly, we begin deriving our approximate model by defining the sequence of random vectors
\beq
\mathbf{Z}^{n}=\begin{bmatrix}
x_{1}^{n}, &
x_{2}^{n}, &
w_{1}^{n}, &
w_{2}^{n}
\end{bmatrix}^{T},
\eeq
for $n\geq 0$.
We can exactly express the evolution of $\mathbf{Z}^n$ as a random dynamical system
\[
\mathbf{Z}^{n+1}=\Phi(\mathbf{Z}^{n},z_{n})
\]
where we see that each $\mathbf{Z}^{n+1}$ is determined as a function of the previous value $\mathbf{Z}^n$ and the random variable $z_n$. Hence the randomness of $\mathbf{Z}^{n+1}$ only enters through $z_n$ once $\mathbf{Z}^n$ is known. The expression for $\Phi$ is
\beq
\Phi(\mathbf{Z}^{n},z_{n})=\begin{cases}
\begin{bmatrix}
\brac{\dfrac{x_{1}^{n} w_{1}^{n} e^{-\lambda }+z_{n}}{w_{1}^{n}e^{-\lambda }+1}  }, &
x_{2}^{n}, &
\brac{w_{1}^{n}e^{-\lambda }+1}, &
w_{2}^{n}e^{-\lambda }
\end{bmatrix}^{T} & \mbox{if }z_{n}\leq \dfrac{x_{1}^{n}+x_{2}^{n}}{2} \\ \\

\begin{bmatrix}
x_{1}^{n}, &
\brac{ \dfrac{x_{2}^{n} w_{2}^{n} e^{-\lambda }+z_{n}}{w_{2}^{n}e^{-\lambda }+1}  }, &
w_{1}^{n}e^{-\lambda }, &
\brac{w_{2}^{n}e^{-\lambda }+1}
\end{bmatrix}^{T} & \mbox{if }z_{n}> \dfrac{x_{1}^{n}+x_{2}^{n}}{2}
\end{cases}.
\eeq

This expression is exact but unwieldy, so we derive an approximate model by linearizing the system about a point where we expect the invariant measure to be densest. 
Define $F(\mathbf{x})=\expect (\Phi(\mathbf{x},z))$, where $z$ is uniform on $[0,1]$.  
Let $\mathbf{Z}^{*}$ be the vector such that $\mathbf{Z}^{*}=\expect (\Phi(\mathbf{Z}^{*},z))$ \cite{bhatt}.  This can be thought of as a sort of analogue of the fixed point of a dynamical system for our random dynamical system. We linearize the random dynamical system about this point. We can expect our system to be well approximated by the linearized system if fluctuations about $\mathbf{Z}^*$ are not too large.


First, one can determine that $\mathbf{Z}^{*}$ is
\beqa
\mathbf{Z}^{*}=\begin{bmatrix}
1/4, &
3/4, &
W/2, &
W/2
\end{bmatrix}^{T} \label{fixed},
\eeqa
where $W= (1-e^{-\lambda })^{-1}$ as defined in the proof for Lemma \ref{lem:weightbound}.

Define another random variable $\mathbf{y}^{n}=\mathbf{Z}^{n}-\mathbf{Z}^{*}$.  The Jacobian of $F$ at $\mathbf{Z}^{*}$ is $J:=\partial F(\mathbf{Z}^{*})$, and the covariance matrix of the random perturbation at $\mathbf{Z}^{*}$ is $H:=\expect (G(\mathbf{Z}^{*}) G(\mathbf{Z}^{*})^{T})$, where $G(\mathbf{x},z)=\Phi(\mathbf{x},z)-F(\mathbf{x})$.  Let $\mathbf{d}^{n}$, for $n\geq 0$, be an i.i.d. sequence of random variables each distributed as $N(0,H)$.  The $AR(1)$ model of $\mathbf{y}^{n}$ is written
\beqa
\mathbf{y}^{n+1}=J\mathbf{y}^{n}+\mathbf{d}^{n},
\label{ar1}
\eeqa
and is an approximation for the dynamics of $\mathbf{Z}^{n}-\mathbf{Z}^{*}$ \cite{ar1braun,ar1yaffee}.  This dynamical system can also be expressed as $\mathbf{y}^{n+1}=J\mathbf{y}^{n}+H^{1/2}N(0,I)$.  

The matrices $J$, $H$, and $H^{1/2}$, can be found through some tedious calculations to be
\begin{align*}
J&=\begin{bmatrix}
\dfrac{5-e^{-\lambda }}{4(2-e^{-\lambda })} & \dfrac{1-e^{-\lambda }}{4(2-e^{-\lambda })}&0&0 \\ \\
\dfrac{1-e^{-\lambda }}{4(2-e^{-\lambda })} & \dfrac{5-e^{-\lambda }}{4(2-e^{-\lambda })}&0&0 \\ \\
\dfrac{1}{2}&\dfrac{1}{2} & e^{-\lambda } & 0 \\ \\
-\dfrac{1}{2}&-\dfrac{1}{2} & 0 & e^{-\lambda }
\end{bmatrix},
\end{align*}
\begin{align*}
H=\begin{bmatrix}
\dfrac{(1-e^{-\lambda })^{2}}{24(2-e^{-\lambda })^{2}} & 0&0&0 \\ \\
0&\dfrac{(1-e^{-\lambda })^{2}}{24(2-e^{-\lambda })^{2}}&0&0 \\ \\
0& 0& \dfrac{1}{4}&-\dfrac{1}{4} \\ \\
0& 0& -\dfrac{1}{4}&\dfrac{1}{4} 
\end{bmatrix},
\end{align*}
and
\beq
H^{1/2}=\dfrac{1}{2\sqrt{2}} \begin{bmatrix}
\dfrac{(1-e^{-\lambda })}{\sqrt{3}(2-e^{-\lambda })} & 0&0&0 \\ \\
0&\dfrac{(1-e^{-\lambda })}{\sqrt{3}(2-e^{-\lambda })}&0&0 \\ \\
0& 0& 1&-1 \\ \\
0& 0& -1&1 
\end{bmatrix}.
\eeq 

An $AR(1)$ model for the perceptual boundary $b^{n}$ can be derived using the $AR(1)$ process we have found for random variable $\mathbf{y}^{n}$.  First, we calculate the first $2$ components of $\mathbf{y}^{n}$ using Equation~\ref{ar1},
\beq
\begin{bmatrix}
x_{1}^{n+1}-\dfrac{1}{4} \\ \\
x_{2}^{n+1} -\dfrac{3}{4}
\end{bmatrix} =\begin{bmatrix}
\dfrac{5-e^{-\lambda   { }}}{4(2-e^{-\lambda   { }})}\brac{x_{1}^{n}-\dfrac{1}{4} } +\dfrac{1-e^{-\lambda   { }}}{4(2-e^{-\lambda   { }})}\brac{x_{2}^{n}-\dfrac{3}{4}}+\dfrac{1-e^{-\lambda   { }}}{2\sqrt{6}(2-e^{-\lambda   { }})}N(0,1)  \\ \\
\dfrac{5-e^{-\lambda   { }}}{4(2-e^{-\lambda   { }})}\brac{x_{2}^{n}-\dfrac{3}{4}} +\dfrac{1-e^{-\lambda   { }}}{4(2-e^{-\lambda   { }})}\brac{x_{1}^{n}-\dfrac{1}{4} }+\dfrac{1-e^{-\lambda   { }}}{2\sqrt{6}(2-e^{-\lambda   { }})}N(0,1)
\end{bmatrix}.
\eeq
Noting that $b^{n}-1/2=\sqbrac{\brac{x_{1}^{n}-1/4 }+\brac{x_{2}^{n}-3/4 } }/2$, we add the two components of this vector
  together and divide by 2 to get
\beq
b^{n+1}-\dfrac{1}{2}=\dfrac{1}{2}\brac{\dfrac{3-e^{-\lambda   { }}}{2-e^{-\lambda   { }}}}\brac{b^{n}-\dfrac{1}{2} }+\dfrac{1}{4\sqrt{3}}\brac{\dfrac{1-e^{-\lambda   { }}}{2-e^{-\lambda   { }}}} N(0,1).
\eeq
This is an $AR(1)$ model for the perceptual boundary, it can be expressed as
\begin{eqnarray}
Y^{n+1}=KY^{n}+\sigma \eta_{n},
\label{boundarybeh}
\end{eqnarray}
where $Y^{n}=b^{n}-1/2$, the variable $\eta_{n}$ is a noise term drawn from the standard normal distribution $N(0,1)$, and
\[
K=(3-e^{-\lambda   { }})\sqbrac{2(2-e^{-\lambda   { }})}^{-1}, \ \ \ \ \ \sigma=\brac{1-e^{-\lambda   { }}}\sqbrac{4\sqrt{3}(2-e^{-\lambda   { }})}^{-1}.  
\]
Equation~\ref{boundarybeh} is our simplified probabilistic model for the motion of the boundary $b^n - 1/2$.

As $\lambda\rightarrow 0$, the variables $K$ and $\sigma$ approach $1$ and $0$ respectively, meaning the stochastic process approaches the case where $b^{n+1}=b^{n}$.   This fits with the fact that in when $\lambda=0$, our original model reverts to the MacQueen model, in which there are no fluctuations as $n \rightarrow \infty$.

In order to assess the quality of the AR(1) process $Y^n$ as a model for $b^n$, we compare the variance of $Y^n$ with the variance of $b^n$ while varying $n$ and $\lambda$. For the process $Y^n$, \cite{ar1varref} provides an exact formula for the variance as a function of time when $Y^0=0$:
\begin{align*}
\Var [Y^n] &= \mathbb{E}(Y^{n})^{2}=\sum_{j=0}^{n-1}K^{2j}\sigma^{2} \\
&=\dfrac{1}{48}\left(\dfrac{1-e^{-\lambda  { }}}{2-e^{-\lambda  { }}} \right)^{2} ~ \sum_{j=0}^{n-1}\left(  \dfrac{3-e^{-\lambda   { }}}{2(2-e^{-\lambda   { }})} \right)^{2j}.
\end{align*}

For the original process $b^n$, we estimate the variance using Monte Carlo simulation.
We perform a large number $N$ of simulations of the exemplar system. In each case we simulated the system starting from 
 the  deterministic initial condition $\mathbf{Z}^0 = \mathbf{Z}^*$. This meant that $\mathbb{E} b^n = b^0$ for all $n$.
For each Monte Carlo simulation, we simulated the system for $n= 0, \ldots, \lceil 400/\lambda \rceil$, since this gave convergence to the equilibrium distribution of $\{b_n\}_{n\geq 0}$ for each $\lambda$. For each of a range of $n$ and $\lambda$,
Figure \ref{fig:bdryvar} shows a comparison between the variance of $b^{n}$ (calculated from simulations of the model) in green and blue, and the variance of the AR$(1)$ approximation $Y^{n}$ in red.  The number of Monte Carlo samples $N$ was chosen large enough so that the statistical error in the plot was negligible compared to the thickness of the lines.  

\begin{figure}[h!]
	\includegraphics[width=13cm]{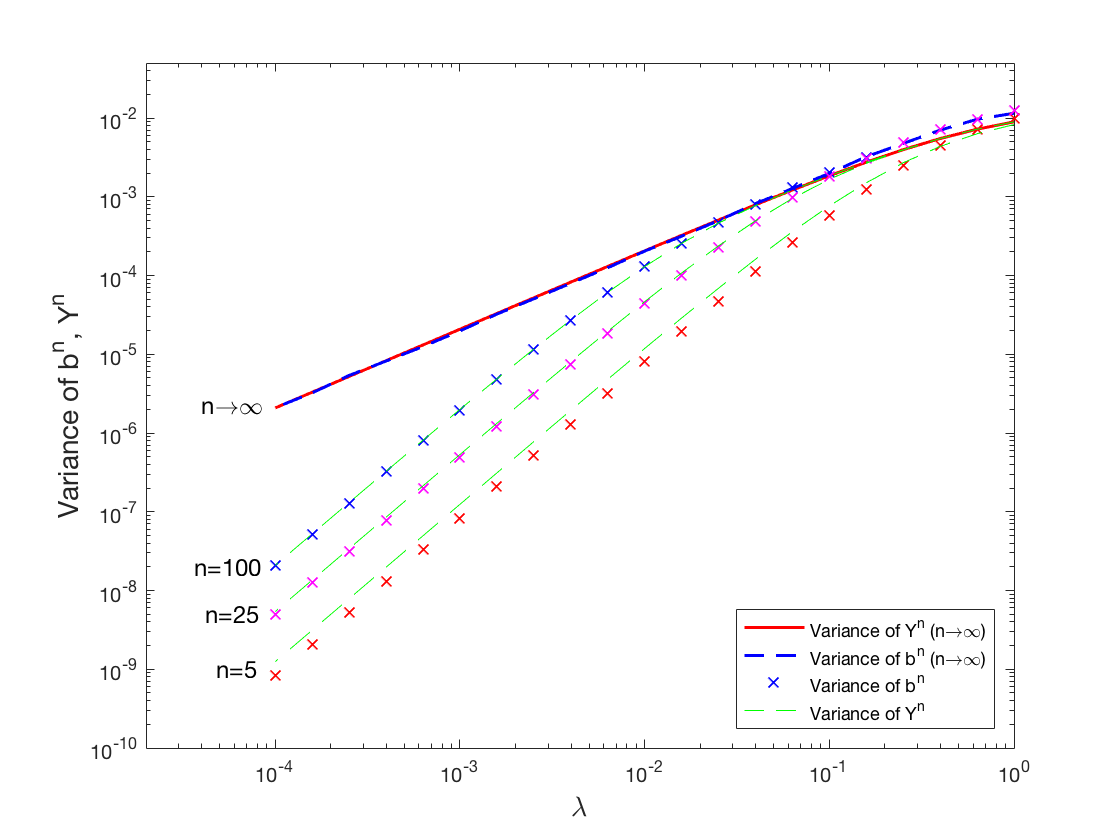}
	\caption{The variances of $b^{n}$ and $Y^n$ versus $\lambda$ for various times $n$.  We set $x_{1}^{0}=1/4$, $x_{2}^{0}=3/4$, and $w_{1}^{0}=w_{2}^{0}=W/2$, so that we start at the fixed point of the system as given in Equation~\ref{fixed}.  The $n\rightarrow \infty$ case for $b^n$ is approximated by setting $n=\lceil 400/ \lambda \rceil$.}
	\label{fig:bdryvar}
\end{figure}

As expected, for all $\lambda$ the variances of $b^n$ and $Y^n$ increase up to their respective equilibrium values as $n \rightarrow \infty$. For smaller values of $\lambda$, the equilibrium values of the variances of $b^n$ and $Y^n$ are indistinguishable to within the accuracy we compute them here. However, for larger values of $\lambda$, the equilibrium variance of the AR(1) model significantly underestimates the equilibrium variance of the original model, showing its limitation as an approximate model.


\bibliographystyle{siamplain}
\bibliography{bensbib}{}

\def\cprime{$'$} \def\cprime{$'$}
\begin{thebibliography}{10}

\bibitem{bhatt}
{\sc R.~Bhattacharya and M.~Majumdar}, {\em Random dynamical systems: theory
  and applications}, Cambridge University Press, 2007.

\bibitem{voronoi}
{\sc Q.~Du, V.~Faber, and M.~Gunzburger}, {\em Centroidal voronoi
  tessellations: Applications and algorithms}, SIAM Rev., 41 (1999),
  pp.~637--676.

\bibitem{jaeger1}
{\sc G.~J{\"a}ger}, {\em Applications of game theory in linguistics}, Language
  and Linguistics Compass, 2 (2008), pp.~406--421.

\bibitem{jaeger2}
{\sc G.~J{\"a}ger, L.~P. Metzger, and F.~Riedel}, {\em Voronoi languages:
  Equilibria in cheap-talk games with high-dimensional types and few signals},
  Games and economic behavior, 73 (2011), pp.~517--537.

\bibitem{macqueen}
{\sc J.~MacQueen}, {\em Some methods for classification and analysis of
  multivariate observations}, in Proceedings of the fifth Berkeley symposium on
  mathematical statistics and probability, Oakland, CA, USA., 1967,
  pp.~281--297.

\bibitem{madsen2007}
{\sc H.~Madsen}, {\em Time series analysis}, CRC Press, 2007.

\bibitem{ar1braun}
{\sc J.~Maindonald and W.~J. Braun}, {\em Data Analysis and Graphics Using R},
  Cambridge University Press, third~ed., 2010.
\newblock Cambridge Books Online.

\bibitem{pierrehumbert}
{\sc J.~B. Pierrehumbert}, {\em Exemplar dynamics: word frequency, lenition,
  and contrast}, Frequency and the emergence of linguistic structure, 45
  (2001), pp.~137--157.

\bibitem{ar1varref}
{\sc D.~Ruppert}, {\em Statistics and Data Analysis for Financial Engineering},
  Springer Texts in Statistics, Springer, 2010.

\bibitem{tupper2015}
{\sc P.~F. Tupper}, {\em Exemplar dynamics and sound merger in language}, SIAM
  Journal on Applied Mathematics, 75 (2015), pp.~1469--1492.

\bibitem{wedel2012}
{\sc A.~Wedel}, {\em Lexical contrast maintenance and the organization of
  sublexical contrast systems}, Language and Cognition, 4 (2012), pp.~319--355.

\bibitem{wedel2006}
{\sc A.~B. Wedel}, {\em Exemplar models, evolution and language change}, The
  Linguistic Review,  (2006), pp.~247--274.

\bibitem{ar1yaffee}
{\sc R.~A. Yaffee and M.~McGee}, {\em Introduction to Time Series Analysis and
  Forecasting: With Applications of SAS and SPSS}, Academic Press, Inc.,
  Orlando, FL, USA, 1st~ed., 2000.

\end{thebibliography}

\end{document}